\documentclass[nohyperref]{article}

\usepackage{microtype}
\usepackage{graphicx}
\usepackage{subfigure}
\usepackage{booktabs} 

\usepackage{hyperref}

\usepackage[accepted]{icml2022}

\usepackage{amsmath}
\usepackage{amssymb}
\usepackage{mathtools}
\usepackage{amsthm}

\usepackage[capitalize,noabbrev]{cleveref}

\theoremstyle{plain}
\newtheorem{theorem}{Theorem}[section]

\newtheorem{lemma}[theorem]{Lemma}

\theoremstyle{definition}

\theoremstyle{remark}

\usepackage[textsize=tiny]{todonotes}

\usepackage{amsmath,amscd,amsbsy,amssymb,latexsym,url,bm,amsthm}
\allowdisplaybreaks
\usepackage{cleveref}
\usepackage{verbatim}
\usepackage{color}
\usepackage{dsfont}

\newcommand{\cF}{\mathcal{F}}

\newcommand{\radius}{\mathrm{Radius}}

\newcommand{\abs}[1]{\left| #1 \right|}
\newcommand{\bOne}[1]{\mathds{1} \! \left\{#1\right\}}
\newcommand{\bracket}[1]{\left(#1\right)}

\newcommand{\set}[1]{\left\{ #1 \right\}}
\newcommand{\EE}[1]{\mathbb{E} \left[#1\right]}

\newif\ifsup\supfalse
\suptrue

\DeclareMathOperator*{\argmax}{argmax}
\DeclareMathOperator*{\argmin}{argmin}

\icmltitlerunning{Simultaneously Learning Stochastic and Adversarial Bandits with General Graph Feedback}

\begin{document}

\twocolumn[
\icmltitle{Simultaneously Learning Stochastic and Adversarial Bandits with\\ General Graph Feedback
}



\icmlsetsymbol{equal}{*}

\begin{icmlauthorlist}
\icmlauthor{Fang Kong}{yyy}
\icmlauthor{Yichi Zhou}{comp}
\icmlauthor{Shuai Li}{yyy}
\end{icmlauthorlist}

\icmlaffiliation{yyy}{John Hopcroft Center for Computer Science, Shanghai Jiao Tong University, Shanghai, China}
\icmlaffiliation{comp}{Microsoft Research Asia, Beijing, China}

\icmlcorrespondingauthor{Shuai Li}{shuaili8@sjtu.edu.cn}

\icmlkeywords{Machine Learning, ICML}

\vskip 0.3in
]



\printAffiliationsAndNotice{}  


\begin{abstract}
The problem of online learning with graph feedback has been extensively studied in the literature due to its generality and potential to model various learning tasks. Existing works mainly study the adversarial and stochastic feedback separately. If the prior knowledge of the feedback mechanism is unavailable or wrong, such specially designed algorithms could suffer great loss. To avoid this problem, \citet{erez2021towards} try to optimize for both environments. However, they assume the feedback graphs are undirected and each vertex has a self-loop, which compromises the generality of the framework and may not be satisfied in applications. With a general feedback graph, the observation of an arm may not be available when this arm is pulled, which makes the exploration more expensive and the algorithms more challenging to perform optimally in both environments. In this work, we overcome this difficulty by a new trade-off mechanism with a carefully-designed proportion for exploration and exploitation. We prove the proposed algorithm simultaneously achieves $\mathrm{poly} \log T$ regret in the stochastic setting and minimax-optimal regret of $\tilde{O}(T^{2/3})$ in the adversarial setting where $T$ is the horizon and $\tilde{O}$ hides parameters independent of $T$ as well as logarithmic terms. To our knowledge, this is the first best-of-both-worlds result for general feedback graphs.

\end{abstract}


\section{Introduction}

\begin{table*}[htbp]
\centering
\begin{tabular}{lll}
\toprule
      & Regret bound (stochastic) & Regret bound (adversarial) \\ \hline
\rule{0pt}{15pt} \citet{wu2015online} &    $O\bracket{|D|\log T/\Delta^2}$,   $\Omega \bracket{|D|\log T/\Delta^2}$                                                              & -                                                                    \\ \hline
\rule{0pt}{15pt} \citet{alon2015online} & -                                                                   &  $O\bracket{(|D|\log K)^{\frac{1}{3}}T^{\frac{2}{3}}}$, $\Omega\bracket{ (|D|/\log^2 K)^{\frac{1}{3}}T^{\frac{2}{3}} }$
\\ \hline
\rule{0pt}{15pt} \citet{chen2021understanding}  & -                                                                   &  $O\bracket{(\delta^*\log K)^{\frac{1}{3}}T^{\frac{2}{3}}}$ ,  $\Omega\bracket{ (\delta^*/\alpha)^{\frac{1}{3}}T^{\frac{2}{3}} }$                                                        \\ \hline
\rule{0pt}{15pt} Ours  &   $O\bracket{ |D|^2 \bracket{\log T/\Delta^2}^{\frac{3}{2}} }$                                          &    $O\bracket{|D|^{\frac{1}{3}}K^{\frac{2}{3}}T^{\frac{2}{3}}\sqrt{\log T}}$                                                                  \\ \bottomrule
\end{tabular}
\caption{Comparisons of regret bounds with most related works in different environments for online learning with general graph feedback. $T$ is the horizon, $K$ is the number of arms, $\Delta$ is the minimum reward gap between the optimal and sub-optimal arms, $D$ is the dominating set of the feedback graph, $\delta^*$ is the fractional weak domination number with $\delta^*\le|D|$ and $\alpha=O(K/\delta^*)$ is the integrality gap of the linear program for vertex packing. The regret bounds in \citet{wu2015online} do not show explicit dependence on $D$ and $\Delta$. We compute the order of them in the bar graph which consists of $K/2$ disjoint undirected edges which connect $K$ arms. }
\label{table:comparison}
\end{table*}

The online learning problem can be formulated by a repeated game between the learner and an unknown environment \cite{cesa2006prediction,lattimore2020bandit}. 
The environment contains $K$ arms. 
In each round, the learner selects one among these arms and observes some feedback. 
At the end of the round, it obtains the reward of the selected arm. 
The objective of the learner is to minimize the cumulative regret over a specified horizon, defined as the difference between the cumulative reward of the optimal arm and that of the arms selected.

To achieve this goal, the learner should adjust its strategy according to the feedback it observes. One of the most widely-studied feedback models is \textit{bandit feedback}, where the learner only observes the reward of the selected arm in each round \cite{auer2002finite,lattimore2020bandit}. 
The other extreme is called \textit{full feedback}, where the learner is able to observe the reward of all arms no matter which arm it selects \cite{cesa2006prediction,eban2012learning}. 
Both these two extremes are special cases of a general feedback model, where the feedback is characterized by a directed graph $G=([K],E)$ \cite{mannor2011bandits}. An edge $(i,j) \in E$ implies that the learner can observe arm $j$ when selecting arm $i$.
Under this framework, the graph of bandit feedback is composed of $K$ disjoint self-loops and that of full feedback is a clique with self-loops. 
Benefiting from its generality and potential to model various learning tasks, graph feedback has been widely studied in the literature \cite{wu2015online,alon2015online,li2020stochastic,chen2021understanding}. 

Previous works on graph feedback mainly study two standard environments: stochastic \cite{buccapatnam2014stochastic,wu2015online,tossou2017thompson,li2020stochastic} and adversarial \cite{mannor2011bandits,alon2015online,chen2021understanding}. In the stochastic setting, the rewards of an arm are drawn independently from a fixed distribution. While in the adversarial setting, the rewards can be chosen arbitrarily. 
All of these works consider observable feedback graphs, i.e., each arm can be observed by selecting some arms. 
It is not hard to see that observability is necessary to guarantee sub-linear regret \cite{alon2015online}. 
With a general observable graph, the optimal regret in both settings has been well understood. 
In the stochastic setting, \citet{wu2015online} develop the instance-optimal regret of order $\Theta(\log T)$. And in the adversarial setting, \citet{alon2015online,chen2021understanding} show the minimax-optimal regret is $\Theta(T^{2/3})$.

However, all these works treat different environment types separately. 
Once the prior knowledge of the reward type is unknown or wrong, the algorithms can suffer great loss. 
So a natural question is whether there is an algorithm that can simultaneously achieve $\mathrm{poly}\log T$ regret in the stochastic setting and $\tilde{O}(T^{2/3})$ regret in the adversarial setting. 

Similar best-of-both-worlds (BoBW) questions have been extensively studied under the bandit feedback \cite{bubeck2012best,seldin2014one,auer2016algorithm,seldin2017improved,zimmert2019optimal} and a more general undirected feedback graph with self-loops \cite{erez2021towards}. 
But as shown in the literature \cite{wu2015online,alon2015online}, even in a known single-type environment, learning with general feedback graphs (possibly without self-loops) is much harder. 
It is because the observations on an arm may not be available by simply selecting it, which makes the exploration of arms more expensive and the exploration-exploitation trade-off more challenging. 
This problem becomes especially difficult when trying to optimize for two different environments. 
Specifically, to achieve instance-optimality, the algorithm tends to first select arms with more observations to identify the optimal one. But these arms may have low rewards, which destroys the minimax-optimality over a specified horizon. 
So despite the importance of the above BoBW results under special feedback graphs, it is still open whether such type of results can be obtained with general graph feedback.

In this paper, we make progress on this problem. 
We develop a new mechanism to trade off exploration and exploitation with carefully-designed proportions for these two parts. 
Equipped with this mechanism and the carefully-designed testing for the environment type, the algorithm is proved to achieve $\mathrm{poly}\log T$ regret in the stochastic setting and $\tilde{O}(T^{2/3})$ regret in the adversarial setting. 
Table \ref{table:comparison} compares the regret bounds with most related works. It is shown that, compared with the state-of-the-art algorithms in both settings, our algorithm only suffers from additional logarithmic and constant factors in the regret while without prior knowledge of the environment type.


\section{Related Work}

The study of online learning with bandit feedback has a long history
\cite{thompson1933likelihood,lai1985asymptotically}. This problem has been well studied in both stochastic \cite{auer2002finite} and adversarial setting \cite{auer2002nonstochastic}.
To generalize the feedback type, 
\citet{mannor2011bandits} first introduce the framework of graph feedback with self-loops in the adversarial setting, which problem has been extensively studied by the following works \cite{kocak2014efficient,cohen2016online}. 
\citet{alon2015online} successfully remove the assumption of self-loops. They develop the Exp3.G algorithm and establish theoretical guarantees for graphs with different observabilities. 
To be specific, the optimal regret is $\Theta(T)$ for non-observable graphs, $\tilde{\Theta}(\sqrt{T})$ for strongly-observable graphs, and $\tilde{\Theta}(T^{2/3})$ for general observable graphs, where a strongly-observable graph requires each arm to have a self-loop or have in-edges from all of the other arms. 
The $\tilde{\Theta}(T^{2/3})$ regret also indicates the hardness to learn with a general feedback graph, which has been recently improved by \citet{chen2021understanding} in terms of the dependence on graph structure. 
In the stochastic setting, \citet{wu2015online} first consider general observable graphs and show the optimal regret is of order $\mathrm{poly} \log T$. \citet{li2020stochastic} later generalize this result by considering probabilistic graph feedback.

All of the above works separately study different environment types. A natural question is whether an algorithm can be derived to achieve the BoBW results. 
The BoBW problem was first studied by \citet{bubeck2012best} in the bandit setting. 
They propose an algorithm and show it achieves not only $\mathrm{poly}\log T$ regret in the stochastic environment but also $\tilde{O}(\sqrt{T})$ regret in the adversarial environment.  
Their algorithm is based on several consistency conditions to detect the optimality of the arms and the type of the environment, which is later refined by \citet{auer2016algorithm}. 
A similar mechanism has also been adopted to obtain the BoBW result for linear bandit \cite{lee2021achieving}. 
The other line of BoBW works are variants of online mirror descent (OMD) \cite{zimmert2019optimal}, which achieves better results than \citet{bubeck2012best,auer2016algorithm} in terms of logarithmic factors. 
Apart from bandit feedback, the BoBW problem has been recently studied with more general feedback graphs, though the graphs are assumed to be undirected and contain a self-loop on each arm \cite{erez2021towards}.  
The algorithm in this work is a variant of OMD, which achieves $\mathrm{poly}\log T$ regret in the stochastic setting and $O(\sqrt{T})$ regret in the adversarial setting.
Recently, \citet{lu2021corruption} study a model between stochastic and adversarial, where the rewards are generated stochastically with a small fraction being corrupted adversarially. 
Their work is also only applicable to undirected feedback graphs with self-loops. 
To the best of our knowledge, we are the first to give the BoBW results for online learning with general graph feedback.

\section{Setting}

The problem can be characterized by a graph $G=(V,E)$, where $V=\set{1,2,\ldots,K}$ is the set of $K$ arms and $E=\set{(i,j)}$ is the set of directed edges between arms. For each arm $i$, denote $N^{\mathrm{in}}(i) = \set{j\in V: (j,i)\in E}$ as the set of in-neighbors and $N^{\mathrm{out}}(i) = \set{j\in V: (i,j)\in E}$ as the set of out-neighbors of $i$.
 In each round $t = 1,2\ldots$, the learner selects an arm $I_t \in V$ and the environment provides a reward vector $r_t := (r_t(1),r_t(2),\ldots,r_t(K) ) \in [0,1]^K$. 
The learner is then rewarded with $r_{t}(I_t)$ and observes $r_{t}(j)$ for each arm $j\in N^{\mathrm{out}}(I_t)$.




It is not hard to see that if some arm $i$ is unobservable, i.e.,  $N^{\mathrm{in}}(i)=\emptyset$, the learner could not determine which arm is optimal. 
Therefore, we consider a general observable graph, i.e.,  $N^{\mathrm{in}}(i)\neq\emptyset$ for each arm $i\in V$. 
It is worth noting that in such a general graph, there may exists arm $i$ with $(i,i)
\notin E$. In this case, the learner though obtains the reward $r_{t}(i)$ by selecting arm $i$, has no chance to get the observation on it. 



Stochastic and adversarial are two standard environments studied in the literature \cite{wu2015online,alon2015online,li2020stochastic,chen2021understanding}. 
In the stochastic setting, the reward vector $r_t$ at time $t$ is independently drawn from a fixed reward distribution with expectation $\EE{r_t} = \mu := (\mu_1,\mu_2,\ldots, \mu_K)$. Denote $i^* = \argmax_{i\in V}\mu_i$ as the optimal arm, which we assume is unique as previous works \cite{wu2015online,li2020stochastic,erez2021towards}. And for each arm $i\in V$, define $\Delta_i = \mu_{i^*} - \mu_i$ as the expected reward gap of $i$ compared with $i^*$. 
In this setting, the pseudo regret is defined as \begin{align}
     Reg(T) = \sum_{t=1}^T \mu_{i^*} - \mu_{I_t}\,.\label{eq:def:regret:sto}
\end{align}
The instance-optimal expected regret attracts more interest in this setting, which is shown to be of order $\Theta(C(\mu)\log T)$ \cite{wu2015online}, where $C(\mu)$ is a constant related to the specific problem instance. 
The adversarial setting is more general than the stochastic setting, where the reward vector $r_t$ can be chosen arbitrarily. 
Following previous works \cite{alon2015online,chen2021understanding}, we consider the oblivious adversary where the rewards are chosen at the start of the game. 
The pseudo regret in this setting is defined as
\begin{align}
    Reg(T) = \max_{i\in V}\sum_{t=1}^T r_{t}(i) - \sum_{t=1}^T r_{t}(I_t) \,.  \label{eq:def:regret:adv}
\end{align}
Researchers are interested in the minimax-optimal expected regret in this setting, which is of order $\Theta(T^{2/3})$ \cite{alon2015online,chen2021understanding}.








\begin{algorithm*}[thb!]
    \caption{BoBW with General Graph Feedback}\label{alg}
    \begin{algorithmic}[1]
    \STATE Input: graph $G=(V,E)$, dominating set $D$, hyper-parameter $\set{\gamma_t}_{t}$ \label{alg:input}\\
    \STATE Initialize: active arm set $A \gets V$; $\forall i \in V, \tau_i \gets \infty, \tau'_i \gets \infty$;  \\
    active dominating set $D_{A}\gets D$; $\forall i \in D, \tau_i^D \gets \infty, u_{i} \gets \frac{1}{|D|}$ \label{alg:initial}\\
    \FOR{$t=1,2,\ldots$}
        \FOR{$i\in V$}
            \STATE $p_{t,A}(i) \gets \bOne{i \in A}\cdot \frac{1}{|A|}$ \label{alg:exploit:probability}
            \STATE $p_{t,D}(i) \gets\bOne{i \in D_{A}}  \cdot \bracket{1 - \sum_{j\in D\setminus D_A} \frac{u_j \tau_j^D}{t} }\frac{1}{|D_A|}  +\bOne{i \in D\setminus D_A}\cdot  \frac{u_i \tau_i^D}{t}$ \label{alg:explore:probability}
            \STATE $p_{t}(i) \gets (1-\gamma_t)\cdot p_{t,A}(i) + \gamma_t\cdot p_{t,D}(i)$  \label{alg:select:probability}
        \ENDFOR
        \STATE Select $I_t \sim p_t$ and be rewarded with $r_t(I_t)$; observe the reward $\set{(i,r_{t}(i)), \forall i \in N^{\mathrm{out}}(I_t) )}$ \label{alg:select and obtain}
        \STATE For each $i \in V$, compute
        \begin{align}
            \tilde{H}_t(i) \gets \frac{1}{t}\sum_{s=1}^t  \frac{r_{t}(i)\cdot \bOne{i\in N^{\mathrm{out}}(I_t)}}{\sum_{j\in N^{\mathrm{in}}(i) } p_{t}(j) } \,. \label{eq:alg:defH}
        \end{align}
        \FOR{$i \in A$ such that \begin{align}
            &\tilde{H}_{t}(j') - \tilde{H}_{t}(i) > 5\radius_t(j')+3\radius_t(i) \text{ for } j'\in \argmax_{j\in A}\tilde{H}_{t}(j)  \label{eq:alg:delete:condition}\\
            &\text{where } \mathrm{Radius}_t(j)=\sqrt{ 4\bracket{ \frac{|D|}{t^2}\sum_{s=1}^{\min\set{t,\tau'_j}} \frac{1}{\gamma_s} + \frac{ |D| \cdot \max\set{t-\tau'_j, 0} }{\gamma_t \tau'_j t }  }\log \frac{t}{\delta} + \frac{5|D|^2 }{\gamma_t^2 \min\set{t^2, {\tau'_j}^2 } }\log^2 \frac{t}{\delta} }, \forall j \label{eq:alg:def:radius}
        \end{align} } 
        \STATE $A \gets A\setminus\set{i}, \tau_i \gets t $  \label{alg:delete arm from A}
        \ENDFOR
        \FOR{$j\in D_A$ such that $i\notin A, \forall i \in N^{\mathrm{out}}(j)$ or $|A|=1$}\label{alg:delete arm from DA:start}
            \STATE $D_A \gets D_A\setminus\set{j}, \tau_j^D \gets t, u_j \gets p_{t,D}(j)$ \label{alg:delete arm from DA}
        \ENDFOR
        \FOR{$i\in V$ such that $ \tau'_i >t$ and $\tau_j^D \le t, \forall j\in D$ with $i\in N^{\mathrm{out}}(j)$ }
            \STATE $\tau'_i \gets t$ \label{alg:update:tau_prime}
        \ENDFOR \label{alg:update:tau_prime:end}
        \IF{$\exists i \notin A$ such that 
        \begin{align}
            \tilde{H}_{t}(j') - \tilde{H}_{t}(i) \le 3\mathrm{Radius}_t(j')+ \mathrm{Radius}_t(i) \text{ for }  j'\in \argmax_{j\in A}\tilde{H}_{t}(j)  \label{eq:alg:detect:adv}
        \end{align}
        
        }
        \STATE Start Exp3.G and set $\tau \gets t$ \label{alg:detect:adv}
        \ENDIF
    \ENDFOR
    \end{algorithmic}
\end{algorithm*}

\section{Algorithm}


In this section, we present our BoBW algorithm for online learning with general graph feedback (Algorithm \ref{alg}). 

The algorithm takes the feedback graph $G$ and a dominating set $D$ of the graph as input (Line \ref{alg:input}). A dominating set contains arms whose out-neighbors can cover the whole arm set, i.e., $\cup_{j\in D} N^{\mathrm{out}}(j)=V$. Thus any arm can be observed by selecting arms in $D$. 

At a high level, the algorithm starts with the assumption that the environment is stochastic and continuously monitors whether the assumption is satisfied. 
Typical optimal stochastic algorithms \cite{wu2015online,li2020stochastic} first explores arms in the dominating set to collect more observations until the optimal arm is identified, and then focuses on the estimated optimal one. 
But such type of algorithms would fail if the underlying environment is adversarial, since the regret caused by exploring arms cannot be controlled before the environment is detected to be adversarial. 


As in the classical multi-armed bandits (MAB) problem, to minimize the cumulative regret, the learner needs to trade off exploitation and exploration. The former focuses on arms with high observed rewards so far to maintain high profits, while the latter aims to select unfamiliar arms to collect more observations \cite{auer2002finite,lattimore2020bandit}. 
However, with general graph feedback, the observation of an arm may not be available by selecting it. Thus, to collect observations on an arm $i$, the learner needs to select some arms from $N^{\mathrm{in}}(i)$. 
To this end, the algorithm maintains the active arm set $A$ and the active dominating set $D_A$, which are initialized as $V$ and $D$, respectively (Line \ref{alg:initial}). 
Intuitively, $A$ is the candidate set of arms with the potential to be optimal and $D_A$ is a dominating set of $A$.
The learner can then explore arms in $D_A$ to get more observations and exploit arms in $A$ to collect more rewards. How to balance these two parts is the key to achieving optimal regret in both stochastic and adversarial environments. The hyper-parameter $(\gamma_t)_{t}$ is taken as input to control this trade-off (Line \ref{alg:input}).

For the proceeding of the above exploration-exploitation trade-off, the algorithm maintains several indicators for arms. Specifically, for each arm $i\in V$, denote $\tau_i$ as the time when $i$ is considered as sub-optimal and deleted from active set $A$. 
Then if the environment is truly stochastic, there is no need to exploit $i$ anymore after $\tau_i$. 
Similarly, for each arm $j\in D$, let $\tau^D_j$ represent the time when all of $j$'s out-neighbors have been considered as sub-optimal and $j$ is deleted from $D_A$. 
Based on this definition, the learner has no need to explore $j$ for more observations on $N^{\mathrm{out}}(j)$ after $\tau_j^D$.
Besides, for each arm $i\in V$, denote $\tau'_i = \argmin_{j \in D, j\in N^{\mathrm{in}}(i)}\tau^D_j$ as the time when all of $i$'s dominating arms have been deleted from $D_A$. All of above values are initialized as $\infty$ at the beginning (Line \ref{alg:initial}).

To deal with the potential adversary, the learner needs to adopt a randomized strategy. Let $p_t$ be the action distribution at $t$. The probability $p_t(i)$ of selecting $i$ at $t$ comes from two parts: the exploitation probability $p_{t,A}(i)$ (Line \ref{alg:exploit:probability}) and the exploration probability $p_{t,D}(i)$ (Line \ref{alg:explore:probability}). 
Since arms in $V\setminus A$ are considered to be sub-optimal, the learner only needs to exploit arms in $A$. Here the algorithm conducts exploitation through a uniform distribution over active arm set $A$ as Line \ref{alg:exploit:probability}.  
And for arms in the active dominating set $i\in D_A$, the learner also needs to explore $i$ to get more observations on its out-neighbors, which is also implemented by a uniform distribution over $D_A$ (Line \ref{alg:explore:probability}). 
Besides, since the environment still has the potential to be adversarial and the deleted arms may become better, the observations on the deleted arms are still required for monitoring. 
For this reason, the algorithm would re-sample arms $j\in D\setminus D_A$ with a certain probability (Line \ref{alg:explore:probability}).   
Note that the re-sampling probability should be small enough to avoid more regret since $j$ may be sub-optimal arms. 
Here we follow \citet{bubeck2012best} for the re-sampling schedule and continuously decrease the exploration probability for arms $j\in D\setminus D_A$ after $\tau_j^D$ as shown in Line \ref{alg:explore:probability}, where $u_j$ records the exploration probability of dominating arm $j$ at time $\tau_j^D$.  
The probability $p_t(i)$ is then defined as the weighted sum of these two parts and the weight is decided by the hyper-parameter $\gamma_t$ (Line \ref{alg:select:probability}). 

In each round $t$, the learner selects an arm $I_t \sim p_t$ and observes the reward $r_t(i)$ for each arm $i \in N^{\mathrm{out}}(I_t)$ (Line \ref{alg:select and obtain}). 
Based on the collected observations, it constantly detects the optimality of arms, which is implemented relying on the estimator $\tilde{H}_t$ defined as Eq.\eqref{eq:alg:defH}. 
The detection condition Eq.\eqref{eq:alg:delete:condition} is based on the concentration of the constructed estimators, where the concentration radius is shown in Eq.\eqref{eq:alg:def:radius} and the detailed analysis is provided in Section \ref{sec:main:proof}. 
Specifically, if some arm $i$ satisfies the detection condition in Eq.\eqref{eq:alg:delete:condition}, then it is considered to be sub-optimal and would be eliminated from the active arm set (Line \ref{alg:delete arm from A}). 
The active dominating set $D_A$ and the indicator $\tau'_{j}$ for each arm $j$ would be updated accordingly (Line \ref{alg:delete arm from DA:start}-\ref{alg:update:tau_prime:end}). 
It is worth noting that if the optimal arm is identified, i.e., $|A|=1$, $D_A$ should be emptied since there is no need to collect more observations on $A$ (Line \ref{alg:delete arm from DA:start}).



Recall that the learner also needs to re-sample arms $j\in D\setminus D_A$ to collect observations on deleted arms $i\notin A$ to detect whether they become better with the influence of the potential adversary. 
This detection condition should be weak enough to guarantee to be satisfied in the stochastic setting and also strong enough such that before the condition is satisfied, the learner does not pay too much regret in the actual adversarial setting. 
Here we define Eq.\eqref{eq:alg:detect:adv} to check whether a previous sub-optimal arm $i$ gets better. We show that in the stochastic setting, with high probability, this condition does not hold. 
And once it is satisfied, the learner believes the environment is actually adversarial and starts to run the optimal Exp3.G algorithm in the adversarial setting \cite{alon2015online}. 




\section{Theoretical Guarantees and Discussions}


The regret upper bounds of Algorithm \ref{alg} in both stochastic and adversarial environments are presented in Theorem \ref{thm:both:guarantee}.

\begin{theorem}[Main]\label{thm:both:guarantee}
Let $\gamma_t = K^{\frac{2}{3}}|D|^{\frac{1}{3}}t^{-\frac{1}{3}}$. 
In the stochastic setting, Algorithm \ref{alg} guarantees that with probability at least $1-\delta$, the regret is at most $O\bracket{ |D|^2 \cdot \max_{i:i\neq i^*}\bracket{ \frac{\log T/\delta}{\Delta_i^2} }^{\frac{3}{2}}  }$. 
Moreover, in the adversarial setting, with probability at least $1-\delta$, the expected regret is at most $O\bracket{K^{\frac{2}{3}} |D|^{\frac{1}{3}} T^{\frac{2}{3}}\sqrt{\log\frac{T}{\delta}}}$.
\end{theorem}

\subsection{Discussions}\label{sec:discuss}

Theorem \ref{thm:both:guarantee} is the first theoretical result for general graph feedback that simultaneously achieves near-optimal regret in both stochastic and adversarial environments. 
As shown in Table \ref{table:comparison}, our regret upper bound is only $O\bracket{{|D|\sqrt{\log T}}/{\Delta}}$ worse than the optimal regret in the stochastic setting \cite{wu2015online} and $O(K^{2/3}\sqrt{\log T})$ worse than the optimal regret in the adversarial setting \cite{alon2015online,chen2021understanding}. 
Above all, to simultaneously learn in both environments, our algorithm only suffers from additional logarithmic and constant factors in the regret.

Though our adopted framework belongs to the same line as \citet{bubeck2012best}, 
new techniques are needed to deal with the unique challenge brought by general graph feedback. 
With a general feedback graph, the observations and rewards of an arm cannot be simultaneously obtained. Thus an important modification of our algorithm is introducing the exploration on dominating arms and the challenge accompanying it is to trade off between selecting arms in dominating set and active set, which mechanism is controlled by the parameter $\gamma_t$.
It is also worth noting that setting $\gamma_t$ as previous works in one of the separate environments would cause the algorithm to fail in the other one. 
Specifically, if we directly follow the stochastic algorithm \cite{wu2015online} and set $\gamma=1$ before the optimal arm is identified and $\gamma=0$ otherwise before starting Exp3.G, the regret in the adversarial setting would be $O(T)$ since the regret before starting Exp3.G cannot be controlled. 
On the other hand, if we directly follow the adversarial algorithm \cite{alon2015online} and set $\gamma=T^{-1/3}$, the regret in the stochastic setting would be $O(T^{1/3}\log T)$. 
Both these two results are not desirable. 

Intuitively, to achieve the minimax-optimality in the adversarial setting, the learner needs to simultaneously conduct exploration on the dominating set and exploitation on the active set. 
However, if the environment is truly stochastic, the exploitation before an optimal arm is identified would produce additional regret. 
To bound this part of regret, we adjust the exploitation budget to let it polynomial depend on the exploration budget by setting $\gamma_t$ to vary with time $t$. Thus when the optimal arm is identified after $O(\log T)$ exploration rounds, the additional cost of exploitation is also polynomial on $O(\log T)$.
And based on the consideration to minimize the regret in the adversarial setting, the exact value of $\gamma_t$ is set as $O(t^{-1/3})$. 
The additional $\sqrt{\log T}/\Delta$ regret in the stochastic setting compared with the lower bound is just the cost for preserving the minimax-optimality in a harder adversarial environment.

\section{Proof of Theorem \ref{thm:both:guarantee}}\label{sec:main:proof}

Before the main proof, we first introduce some useful notations and inequalities. At any time $t$, let $A_t$ be the active arm set $A$ at the start of round $t$ and denote $H_{t}(i):=\frac{1}{t}\sum_{s=1}^t r_{s}(i)$ as the averaged reward of $i$ at $t$. 

The proof of Theorem \ref{thm:both:guarantee} highly depends on the concentration of the estimator $\tilde{H}_{t}(i)$. In the stochastic setting, it is an estimator of the unknown expected reward $\mu_i$. And in the adversarial setting, it is the estimator of unknown averaged reward $H_{t}(i)$. In these two environments, we want to separately bound the difference between $\tilde{H}_{t}(i)$ and $\mu_i$ ($H_{t}(i)$, respectively). 

\begin{lemma}\label{lem:key:concen}
For any arm $i \in V$ and time $t<\tau$, with probability at least $1-\delta$, $\abs{ \tilde{H}_{t}(i)- \mu_{i}  } \le \radius_t(i)$ in the stochastic setting and $\abs{ \tilde{H}_{t}(i) -H_{t}(i) } \le \radius_t(i)$ in the adversarial setting. 
\end{lemma}

\begin{proof}
Since the proofs of the two inequalities are similar, here we mainly focus on the adversarial setting. 

In each round $t$, let $\cF_t$ be the history of observations. 
Define 
\begin{align*}
    X_{t}(i) = \tilde{r}_{t}(i) - r_{t}(i)=r_{t}(i)\bracket{ \frac{\bOne{i \in N^{\mathrm{out}}(I_t)}}{ \sum_{j\in N^{\mathrm{in}}(i)  } p_{s}(j)  } - 1  } \,. 
\end{align*}
Note that $(X_{s}(i))_{1 \le s \le t}$ is a martingale difference sequences. According to the algorithm, at each time $t<\tau$, 
\begin{align*}
    \sum_{j\in N^{\mathrm{in}}(i)  } p_{t}(j) \ge \frac{\gamma_t}{|D|} \bOne{ t\le \tau'_i } + \frac{\gamma_t \tau'_i}{|D|t} \bOne{ t>\tau'_i } \,.
\end{align*}
Thus $\abs{X_{t}(i)} \le \max\set{ \frac{|D|}{\gamma_t}, \frac{|D| t}{\gamma_t \tau'_i } } \le \frac{|D|}{\gamma_t} \max\set{1, \frac{t}{\tau'_i}} $ and
\begin{align*}
  \sum_{s=1}^t \EE{ X_{s}(i)^2 \mid \cF_{s-1} } \le \sum_{s=1}^{\min\set{t,\tau'_i}}\frac{|D|}{\gamma_s}
+ \sum_{s=\tau'_i}^{t}  \frac{|D| s}{ \gamma_s \tau'_i } \\
\le \sum_{s=1}^{\min\set{t,\tau'_i}}\frac{|D|}{\gamma_s}
+ \frac{t|D|}{\gamma_t \tau'_i}\cdot \max\set{t-\tau'_i, 0}\,,
\end{align*}
where the last inequality holds since $\gamma_t$ decreases with the increase of $t$ by choosing $\gamma_t = (K^2 |D|)^{1/3} t^{-1/3}$. According to Lemma \ref{lem:tech:concen} and the definition of $\radius_t(i)$, we have with probability at least $1-\delta$
\begin{align*}
    \sum_{s=1}^t X_{s}(i) \le t\cdot \radius_t(i)\,,
\end{align*}
and further
\begin{align*}
    \abs{ \tilde{H}_{t}(i) - H_{t}(i) } \le \radius_t(i) \,,
\end{align*}
since $\tilde{H}_{t}(i) - H_{t}(i) = \frac{1}{t} \sum_{s=1}^t X_{s}(i)$. 
\end{proof}

In the following, we provide the analysis of Algorithm \ref{alg} in both settings.

\subsection{Regret Analysis in the Stochastic Setting}\label{sec:proof:sto}
We first simply show that the optimal arm $i^*$ is never deleted from the active arm set $A$. 

\begin{lemma}\label{lem:sto:neverDeleteOptimal}
At any time $t$, $i^*\in A_t$. 
\end{lemma}
\begin{proof}
At any time $t$, for any arm $j \in A_t$,
\begin{align}
    \tilde{H}_{t}(j) - \tilde{H}_{t}(i^*) \le&~ \mu_j - \mu_{i^*} + \radius_t(j)+ \radius_{t}(i^*) \notag \\
   =& -\Delta_j + \radius_t(j)+ \radius_{t}(i^*)\notag \\
    \le & 5\radius_t(j)+ 3\radius_{t}(i^*) \,, \notag
\end{align}
where the first inequality comes from Lemma \ref{lem:key:concen}. Thus the detection condition Eq.\eqref{eq:alg:delete:condition} is never satisfied for $i^*$. 
\end{proof}

We then prove that the Exp3.G algorithm is never started in the stochastic setting. 

\begin{lemma}\label{lem:sto:neverStartExp}
In the stochastic setting, Eq.\eqref{eq:alg:detect:adv} in Algorithm \ref{alg} is never satisfied and thus Exp3.G is never started. 
\end{lemma}

\begin{proof}
For each arm $i$, let $j\in\argmax_{j'\in A_{\tau_i}}\tilde{H}_{\tau_i}(j')$. 
Then according to the definition of $\tau_i$, 
\begin{align}
    \tilde{H}_{\tau_i}(j) - \tilde{H}_{\tau_i}(i) > 5\radius_{\tau_i}(j)+3\radius_{\tau_i}(i) \,.\notag
\end{align}
Besides, according to Lemma \ref{lem:key:concen},
\begin{align}
    \tilde{H}_{\tau_i}(j) - \tilde{H}_{\tau_i}(i) 
    \le&~ \mu_j - \mu_i + \radius_{\tau_i}(j) + \radius_{\tau_i}(i) \notag \\
    \le&~ \Delta_i+\radius_{\tau_i}(j) + \radius_{\tau_i}(i) \,.\notag
\end{align}
Combining the above two inequalities, we have
\begin{align}
    \Delta_i >&~ 4\radius_{\tau_i}(j) + 2\radius_{\tau_i}(i) \notag \\
    =&~ 6\sqrt{\frac{4|D|}{\tau_i^2} \sum_{s=1}^{\tau_i} \frac{1}{\gamma_s} \log\frac{T}{\delta} +\frac{5|D|^2}{\gamma_{\tau_i}^2 {\tau_i}^2}\log^2 \frac{{T}}{\delta}  } \,,
    \label{eq:main:sto:lower:Delta}
\end{align}
where the last equality holds by computing $\radius_{\tau_i}(j)$ and $\radius_{\tau_i}(i)$ with $\tau'_{i}\ge \tau_i$ and $\tau'_{j}\ge \tau_i$. 

Further, at any time $t$, according to Lemma \ref{lem:key:concen} and the fact that $i^*\in A_t$, we have for each arm $i \notin A_t$, 
\begin{align}
    \max_{j\in A_{t}}\tilde{H}_{t}(j) - \tilde{H}_{t}(i) \ge&~ \tilde{H}_{t}(i^*) - \tilde{H}_{t}(i) \notag \\ 
    \ge&~ \Delta_i - \radius_{t}(i^*) - \radius_t(i)  \notag\\
    > &~ 3\radius_{t}(i^*) + \radius_t(i) \label{eq:main:sto:since:valueOfRadius2} \,,
\end{align}
where Eq.\eqref{eq:main:sto:since:valueOfRadius2} holds by computing $\radius_{t}(i^*)$ and $ \radius_t(i)$ and comparing them with $\Delta_i$ which is lower bounded by Eq.\eqref{eq:main:sto:lower:Delta}. 


Let $j_t \in\argmax_{j\in A_t}\tilde{H}_t(j)$.
If $j_t = i^*$, then we have proved Eq.\eqref{eq:alg:detect:adv} in Algorithm \ref{alg} is not satisfied. Otherwise, there must be $\radius_t(j_t)=\radius_t(i^*)$, since both $\tau'_{j_t}>t$ and $\tau'_{i^*}>t$ when the optimal arm is not identified. Above all, no matter in either case, Exp3.G is not started. 
\end{proof}

In the following, we try to bound the number of arms' selections. For convenience, denote $T_i(t) := \sum_{s=1}^t \bOne{I_s=i}$ as the number of $i$'s selections at the end of $t$.

\begin{lemma}\label{lem:sto:pullTime}
For consistency, define $\tau_i^D = 0$ for each arm $i\notin D$. Then at any time $t=1,2\ldots$, with probability at least $1-\delta$, for each arm $i\in V$,
\begin{align*}
    T_i(T) &\le \tau_i+ \sum_{s=1}^{\tau_i^D}\gamma_s +\sum_{s=\tau_i^D}^{T}\gamma_s\frac{\tau_i^D}{s} \\
    +&\sqrt{ 4\bracket{ \tau_i + \sum_{s=1}^{\tau_i^D}\gamma_s + \sum_{s=\tau_i^D}^{T}\gamma_s\frac{\tau_i^D}{s} }\log \frac{T}{\delta} + 5\log^2 \frac{T}{\delta} }\,.
\end{align*}
\end{lemma}


\begin{proof}
Let $Z_{s,i}(p)\sim \mathrm{Bernoulli}(p)$ for time $s=1,2\ldots,t$, $i\in V$. Then $\forall i,t$, $T_i(t) = \sum_{s=1}^t Z_{s,i}(p_{s}(i))$. 
Define $X_{s}(i) = Z_{s,i}(p_{s}(i)) - p_{s}(i) $. We have $\abs{X_{s}(i)} \le 1$ and 
\begin{align*}
    \sum_{s=1}^T \EE{X_{s}(i)^2 \mid \cF_{s-1} }
    \le& \tau_i + \sum_{s=1}^{\tau_i^D}\gamma_s + \sum_{s=\tau_i^D}^{T}\gamma_s\frac{\tau_i^D}{s} \,.
\end{align*}
According to Lemma \ref{lem:tech:concen}, with probability at least $1-\delta$, 
\begin{align*}
    T_i(T)  &\le \tau_i + \sum_{s=1}^{\tau_i^D}\gamma_s + \tau_i^D \sum_{s=\tau_i^D}^{T}\frac{\gamma_s}{s} + \sum_{s=1}^T X_{s}(i)  \\
    &\le \tau_i + \sum_{s=1}^{\tau_i^D}\gamma_s + \sum_{s=\tau_i^D}^{T}\gamma_s\frac{\tau_i^D}{s}  \\
    +&\sqrt{ 4\bracket{ \tau_i + \sum_{s=1}^{\tau_i^D}\gamma_s + \sum_{s=\tau_i^D}^{T}\gamma_s\frac{\tau_i^D}{s} }\log \frac{T}{\delta} + 5\log^2 \frac{T}{\delta} } \,.
\end{align*}

\end{proof}

Based on above results, the pseudo regret $Reg(T) = \sum_{i\neq i^*}\Delta_i T_i(T)$ in the stochastic setting can be upper bounded by
\begin{align*}
    &\sum_{i\notin D,i\neq i^*} \Delta_i \bracket{ \tau_i + \sqrt{ 4\tau_i \log \frac{T}{\delta} + 5\log^2 \frac{T}{\delta}} } \\
    &+ \sum_{i\in D,i\neq i^*} \Delta_i\left( \tau_i+ \sum_{s=1}^{\tau_i^D}\gamma_s  +\sum_{s=\tau_i^D}^{T}\gamma_s\frac{\tau_i^D}{s} \right. \notag\\ 
    &\left.~~+\sqrt{ 4\bracket{ \tau_i + \sum_{s=1}^{\tau_i^D}\gamma_s + \sum_{s=\tau_i^D}^{T}\gamma_s\frac{\tau_i^D}{s} }\log \frac{T}{\delta} + 5\log^2 \frac{T}{\delta} } \right) \,.
\end{align*}
In the following, we separately bound each part. 
\paragraph{Bound $\tau_i$: }
At time $\tau_i-1$, since $i^*\in A_{\tau_i-1}$, it holds that
\begin{align*}
    &\Delta_i - 2\sqrt{ \frac{4|D|}{(\tau_i-1)^2} \sum_{s=1}^{\tau_i - 1} \frac{1}{\gamma_s} \log\frac{\tau_i}{\delta} +\frac{5|D|^2\log^2 \frac{\tau_i}{\delta}}{\gamma_{\tau_i - 1}^2 (\tau_i - 1)^2} }\\
    =&~ \mu_{i^*} - \mu_i - \radius_{\tau_i-1}(i^*)-\radius_{\tau_i-1}(i) \\
    \le&~ \tilde{H}_{\tau_i-1}(i^*)-\tilde{H}_{\tau_i-1}(i) \\
    \le& \max_{j\in A_{\tau_i-1}}\tilde{H}_{\tau_i-1}(j)-\tilde{H}_{\tau_i-1}(i)\\ \le&~ 8\sqrt{ \frac{4|D|}{(\tau_i-1)^2} \sum_{s=1}^{\tau_i - 1} \frac{1}{\gamma_s} \log\frac{\tau_i}{\delta} +\frac{5|D|^2\log^2 \frac{\tau_i}{\delta}}{\gamma_{\tau_i - 1}^2 (\tau_i - 1)^2} } \,,
\end{align*}
where the first equality and the last inequality hold according to the algorithm and the values of $\radius_{\tau_i-1}(i)$ for these different arms, 
the first inequality comes from Lemma \ref{lem:key:concen}. 
We can then conclude $\tau_i \le O\bracket{  \frac{|D|}{K}\bracket{ \frac{\log T/\delta}{\Delta_i^2} }^{3/2}}$.  


\paragraph{Bound $\tau_i^D$ for $i\in D$: }
According to Algorithm \ref{alg}, for each arm $i\in D$, if $i^*\notin N^{\mathrm{out}}(i)$, $\tau_i^D =  \max_{j: j\in N^{\mathrm{out}}(i)}\tau_j$. 
And If $i^*\in N^{\mathrm{out}}(i)$, $\tau_i^D= \max_{j: j\neq i^*}\tau_j $. 

\paragraph{Bound $\sum_{s=1}^{\tau_i^D} \gamma_s$ for $i\in D$: }
\begin{align*}
    \sum_{s=1}^{\tau_i^D} \gamma_s =& \sum_{s=1}^{\tau_i^D} K^{\frac{2}{3}} |D|^{\frac{1}{3}}s^{-\frac{1}{3}} \\
    \le& K^{\frac{2}{3}} |D|^{\frac{1}{3}} \bracket{ 1+\int_{1}^{\tau_i^D} s^{-\frac{1}{3}}ds } \\
    \le& K^{\frac{2}{3}} |D|^{\frac{1}{3}} \bracket{ 1+(\tau_i^D)^{2/3} } 
    \le O\bracket{ K^{\frac{2}{3}} |D|^{\frac{1}{3}}(\tau_i^D)^{2/3} }\,.
\end{align*}

\paragraph{Bound $\tau_i^D\sum_{s=\tau_i^D}^{T} \frac{\gamma_s}{s}  $ for $i\in D$: }

\begin{align*}
     \tau_i^D \sum_{s=\tau_i^D}^{T} \frac{\gamma_s}{s} \le& K^{2/3}|D|^{1/3} \tau_i^D \sum_{s=\tau_i^D}^{T} s^{-4/3} \\
    \le& 4 K^{2/3}|D|^{1/3}\tau_i^D \,.
\end{align*}
Combining the above upper bound for each part, the regret can be bounded by 
\begin{align*}
    &O\left( \sum_{i\neq i^*} \frac{\Delta_i|D|}{K}\bracket{ \frac{\log T/\delta}{\Delta_i^2} }^{\frac{3}{2}}
     + \sum_{i\in D,i\neq i^*} \Delta_i \cdot  \max_{j:j\neq i^*}   \right.  \\
    &\left. \left(  
     \frac{|D|\log T/\delta}{\Delta_j^2} + \frac{|D|^{\frac{4}{3}}}{K^{\frac{1}{3}}}\bracket{ \frac{\log T/\delta}{\Delta_j^2} }^{\frac{3}{2}} \right) \right)\\
    \le &~ O\bracket{ \frac{|D|^{\frac{7}{3}}}{K^{\frac{1}{3}}} \max_{j:j\neq i^*}\bracket{ \frac{\log T/\delta}{\Delta_j^2} }^{\frac{3}{2}}  } \\
    \le& O\bracket{ |D|^2 \max_{j:j\neq i^*}\bracket{ \frac{\log T/\delta}{\Delta_j^2} }^{\frac{3}{2}}  } \,.
\end{align*}
Thus we obtain the upper bound for the pseudo regret in Theorem \ref{thm:both:guarantee}.

\subsection{Regret Analysis in the Adversarial Setting}\label{sec:proof:adv}

Recall that $\tau$ is the time when Exp3.G is started. In the adversarial setting, let $i^* \in \argmax_{i\in V}\sum_{t=1}^\tau r_t(i)$ be one of the optimal arms in the first $\tau$ rounds. 
We first show that $i^*$ would not be deleted from $A$ before $\tau$. 
\begin{lemma}\label{lem:adv:optimalIsNotDelete}
At time $\tau-1, i^* \in A_{\tau-1}$. 
\end{lemma}
\begin{proof}
Denote $I^* \in \argmax_{i\in A_{\tau-1}} \sum_{t=1}^{\tau-1}r_t(i) $ as one of the optimal arms in $A_{\tau-1}$ and let $j\in \argmax_{j'\in A_{\tau-1}}\tilde{H}_{\tau-1}(j)$ be the arm in $A_{\tau-1}$ with the largest estimated reward at $\tau-1$. Then for any arm $i \notin A_{\tau-1}$,
\begin{align*}
     &\sum_{t=1}^{\tau-1} r_t(I^*)-\sum_{t=1}^{\tau-1} r_t(i)\\
    =& (\tau-1)\left( H_{\tau-1}(I^*) - \tilde{H}_{\tau-1}(I^*) + \tilde{H}_{\tau-1}(I^*) \right. \\
    &\left. - \tilde{H}_{\tau-1}(i) + \tilde{H}_{\tau-1}(i) - H_{\tau-1}(i) \right)\\
    \ge & \bracket{\tau-1} \left( -\radius_{\tau-1}(I^*)-\radius_{\tau-1}(i) \right. \\
  &\left.  +\tilde{H}_{\tau-1}(I^*) - \tilde{H}_{\tau-1}(i)\right) \\
  = & \bracket{\tau-1} \left( -\radius_{\tau-1}(I^*)-\radius_{\tau-1}(i) \right. \\
  &\left.  +\tilde{H}_{\tau-1}(I^*) -\tilde{H}_{\tau-1}(j) + \tilde{H}_{\tau-1}(j)- \tilde{H}_{\tau-1}(i)\right) \\
  \ge & \bracket{\tau-1} \left( -\radius_{\tau-1}(I^*)-\radius_{\tau-1}(i) \right. \\
  &\left.  + {H}_{\tau-1}(I^*) -{H}_{\tau-1}(j) + \tilde{H}_{\tau-1}(j)- \tilde{H}_{\tau-1}(i)\right. \\
  &\left. -\radius_{\tau-1}(I^*)-\radius_{\tau-1}(j)  \right)\\
  \ge & \bracket{\tau-1} \left( -2\radius_{\tau-1}(I^*)-\radius_{\tau-1}(i) \right. \\
  &\left.  + \tilde{H}_{\tau-1}(j)- \tilde{H}_{\tau-1}(i) -\radius_{\tau-1}(j)  \right)
  > 0  
\end{align*}
where the penultimate inequality is due to Lemma \ref{lem:key:concen}, the
last inequality holds based on the definition of arm $I^*$, the fact that Eq.\eqref{eq:alg:detect:adv} in Algorithm \ref{alg} is not satisfied at $\tau-1$ and $\radius_{\tau-1}(I^*)=\radius_{\tau-1}(j)$ since both of them are in $A_{\tau-1}$. 
Above all, any previously deleted arm $i \notin A_{\tau-1}$ is worse than $I^*$, thus we can conclude that $i^* \in A_{\tau-1}$ holds. 
\end{proof}

In the following, we bound the averaged regret before $\tau$ compared with $i^*$ caused by selecting different arms. 

\begin{lemma}\label{lem:adv:arm:bound}
For each arm $i$ such that $i \notin A_{\tau-1}$, at $\tau_{i}-1$, 
\begin{align*}
    H_{\tau_i-1}(i^*) - H_{\tau_i-1}(i) \le 10 \radius_{\tau_i-1}(i)\,.
\end{align*}
Similarly, for each arm $i \in A_{\tau-1}$, at time $\tau-1$,
\begin{align*}
    H_{\tau-1}(i^*) - H_{\tau-1}(i) \le 10\radius_{\tau-1}(i) \,.
\end{align*}

\end{lemma}

\begin{proof}
For arm $i \notin A_{\tau-1}$, Eq.\eqref{eq:alg:delete:condition} in Algorithm \ref{alg} does not hold at time $\tau_i-1$, and since $i^*\in A_{\tau_i-1}$, 
\begin{align}
      &~\tilde{H}_{\tau_i-1 }(i^*)-  \tilde{H}_{\tau_i-1 }(i) \notag \\
      \le&\max_{j\in A_{\tau_i-1}} \tilde{H}_{\tau_i-1 }(j)-  \tilde{H}_{\tau_i-1 }(i)\notag \\ 
      \le& ~5\radius_{\tau_i-1}(j)+3\radius_{\tau_i-1}(i)\notag\\
      =& ~8\radius_{\tau_i-1}(i) \label{eq:main:adv:bound:1}
       \,.
\end{align}
where Eq.\eqref{eq:main:adv:bound:1} holds since at $\tau_i-1$, $\radius_{\tau_i-1}(j) = \radius_{\tau_i-1}(i)$. 
Besides, at $\tau_i-1$, $\radius_{\tau_i-1}(i^*) = \radius_{\tau_i-1}(i)$ also holds according to Algorithm \ref{alg}. 
Thus based on Lemma \ref{lem:key:concen}, 
\small
\begin{align}
   &\tilde{H}_{\tau_i-1}(i^*)-  \tilde{H}_{\tau_i-1}(i)  \notag\\
    \ge&  H_{\tau_i-1 }(i^*)-H_{\tau_i-1}(i) - \radius_{\tau_i-1}(i^*) - \radius_{\tau_i-1}(i) \notag\\
    =& H_{\tau_i-1}(i^*)-H_{\tau_i-1}(i) - 2\radius_{\tau_i-1}(i) \,. \label{eq:main:adv:arm2}
\end{align}

\normalsize
Combining Eq.\eqref{eq:main:adv:bound:1} and \eqref{eq:main:adv:arm2}, we can obtain the first inequality in Lemma \ref{lem:adv:arm:bound}. Similar arguments hold at time $\tau-1$ for arm $i \in A_{\tau-1}$ to get the second inequality.  
\end{proof}

Based on the above results, taking the expectation over the randomness of $I_t$, we are able to bound the expected regret during the first $\tau$ rounds as follows. 
\begin{align}
    &\EE{\sum_{t=1}^{\tau} r_{t}(i^*) - r_{t}(I_t) } \notag\\
    =& \sum_{t=1}^{\tau} \bracket{ r_{t}(i^*) - \sum_{i\in V} p_{t}(i) r_{t}(i) }  \notag \\
    =& \sum_{t=1}^{\tau} \bracket{ r_{t}(i^*)  -  \sum_{i\in V} (1-\gamma_{t})p_{t,A}(i)  r_{t}(i) }  \notag \\
   & ~~~~~~~~~~~~ - \sum_{i\in V} \gamma_{t}p_{t,D}(i)  r_{t}(i) \notag\\
    =& \sum_{t=1}^{\tau} \gamma_t \sum_{i\in V} p_{t,D}(i) \bracket{ r_{t}(i^*)-r_{t}(i) } \notag\\ 
    &+  \sum_{t=1}^{\tau} (1-\gamma_t) \sum_{i\in V} p_{t,A}(i) \bracket{ r_{t}(i^*)-r_{t}(i) } \notag\\
    \le& \sum_{t=1}^{\tau} \gamma_t + \sum_{t=1}^{\tau}  \sum_{i\in V} p_{t,A}(i) \bracket{ r_{t}(i^*)-r_{t}(i) } \notag \\
    \le& \sum_{t=1}^{\tau} \gamma_t + \sum_{i\in V} \sum_{t=1}^{\min\set{\tau_i,\tau}}   \bracket{ r_{t}(i^*)-r_{t}(i) } \label{eq:main:adv:since:updateA} \\
    \le & \sum_{t=1}^{\tau} \gamma_t + \sum_{i\in V} \left( \bOne{i\notin A_{\tau}}\cdot \tau_i  \bracket{ H_{\tau_i}(i^*)-H_{\tau_i}(i) } \notag \right.\\ 
    &\left.~~~~~~~~~~~~~~~~~~~+ \bOne{i\in A_{\tau}}\cdot \tau  \bracket{ H_{\tau}(i^*)-H_{\tau}(i) } \right) \notag \\
    \le & \sum_{t=1}^{\tau} \gamma_t + O\bracket{ \sqrt{ K^2|D| \sum_{t=1}^{\tau}\frac{1}{\gamma_t} \log \frac{T}{\delta} } }  \label{eq:main:adv:since:arm:bound} \\
    \le& O\bracket{ K^{2/3}|D|^{1/3}\tau^{2/3}
    \sqrt{\log\frac{T}{\delta} }}\label{eq:main:adv:bound:phase1}  \,,
\end{align}
where Eq.\eqref{eq:main:adv:since:updateA} holds according to Algorithm \ref{alg}, Eq.\eqref{eq:main:adv:since:arm:bound} comes from Lemma \ref{lem:adv:arm:bound} and the exact value of $\radius_{t}(i)$ for $t=\tau_i-1,\tau-1$. The last inequality holds by replacing $\gamma_t$ with $K^{2/3}|D|^{1/3}t^{-1/3}$. 


We can divide the total horizon into two phases and bound the cumulative expected regret over $T$ rounds by $\EE{Reg(T)}\le \EE{\sum_{t=1}^{\tau} r_{t}(i^*) - r_{t}(I_t) } + \EE{\max_{i\in V}\sum_{t=\tau+1}^T r_{t}(i) - r_{t}(I_t) } $. 
The regret in the first phase is upper bounded by Eq.\eqref{eq:main:adv:bound:phase1}. 
For the second phase, as shown in \citet[Theorem 2]{alon2015online}, the regret is of order $O\bracket{ (|D|\log K)^{1/3}T^{2/3} }$. 
Above all, the regret upper bound in Theorem \ref{thm:both:guarantee} is obtained. 


\subsection{Technical Lemma}

\begin{lemma}[Lemma 4.4 in \citet{bubeck2012best}]\label{lem:tech:concen}
Let $\cF_1\subseteq \cF_2\cdots \cF_n$ be filtrations, and $X_1,X_2\cdots, X_n$ be real random variables such that $X_t$ is $\cF_t$ measurable, $\EE{X_t \mid \cF_t} = 0$ and $|X_t|\le b$ for some $b>0$. Denote $V_n = \sum_{t=1}^n \EE{X_t^2 \mid \cF_{t-1}}$ and $\delta>0$. Then with probability at least $1-\delta$, 
\begin{align*}
    \sum_{t=1}^n X_t \le \sqrt{ 4V_n \log \frac{n}{\delta} + 5b^2 \log^2 \frac{n}{\delta}  } \,.
\end{align*}
\end{lemma}

\section{Conclusion}
In this paper, we propose the first algorithm for online learning with general graph feedback that achieves near-optimal regret in both stochastic and adversarial settings. 
Compared with previous BoBW works, we face a more challenging problem where only partial observations are available. 
To simultaneously preserve the minimax-optimality of $O(T^{2/3})$ in a much harder adversarial setting and the instance-optimality in the stochastic setting, we introduce a new trade-off mechanism between exploration and exploitation with carefully designed proportions. 
Without knowing the environment type in advance, the regret upper bound of our proposed algorithm matches the near-optimal results in both settings only up to some logarithmic and constant factors .   

Since the current upper bound in the adversarial setting is on the expected regret, one of the possible directions is to derive a high-probability pseudo regret bound as \citet{bubeck2012best}. 
It is more difficult as the pseudo regret in this problem with general graph feedback can only be decomposed into unobservable quantities, and similar techniques in \citet{bubeck2012best} could not be adopted. 
The other interesting future direction is to investigate possible regularizers for OMD-type algorithms as \citet{zimmert2019optimal,erez2021towards}, which are hoped to achieve tighter regret bounds in both settings.




\bibliography{ref}
\bibliographystyle{icml2022}




\end{document}